\documentclass{article} 
\usepackage{iclr2027_conference,times}
\usepackage[nopatch=eqnum]{microtype}

\usepackage{amsmath,amsfonts,bm}

\def\eqref#1{equation~\ref{#1}}

\def\1{\bm{1}}

\DeclareMathAlphabet{\mathsfit}{\encodingdefault}{\sfdefault}{m}{sl}
\SetMathAlphabet{\mathsfit}{bold}{\encodingdefault}{\sfdefault}{bx}{n}

\usepackage{url}
\usepackage{amsmath,amssymb,amsthm}
\newtheorem{proposition}{Proposition}
\newtheorem{theorem}{Theorem}
\newtheorem{corollary}{Corollary}
\newtheorem{remark}{Remark}
\usepackage{booktabs}
\usepackage{array}
\newcolumntype{L}[1]{>{\raggedright\arraybackslash}p{#1}}
\usepackage{algorithm}
\usepackage[noend]{algpseudocode}
\usepackage{graphicx}
\usepackage{placeins}
\usepackage{capt-of}
\usepackage{hyperref}
\iclrfinalcopy
\hypersetup{
  colorlinks=true,
  linkcolor=blue!50!black,
  citecolor=blue!50!black,
  urlcolor=blue!50!black
}

\title{LLN: Learnable Lens Networks for Parameter-Efficient Long-Horizon Dynamical Prediction}

\author{
Binbin Yong\textsuperscript{1,},
Zhao Su\textsuperscript{1*},
Lan Guo\textsuperscript{1},
Haoran Li\textsuperscript{2}
Jun Shen\textsuperscript{3},
Qi Zhu\textsuperscript{4},
Qingguo Zhou\textsuperscript{1,*}\\[0.5em]
\textsuperscript{1}Lanzhou University,
\textsuperscript{2}Monash University,
\textsuperscript{3}University of Wollongong, \\
\textsuperscript{4}Nanjing University of Aeronautics and Astronautics,
\textsuperscript{*}Corresponding authors
}
\begin{document}

\maketitle
\pagestyle{plain}
\thispagestyle{plain}

\begin{abstract}
Explicit residual connections of the form (x+f(x)), often combined with normalization layers, have become a standard strategy for training very deep neural networks. However, residual addition primarily provides an algebraic shortcut for gradient propagation, while leaving the evolution of feature geometry across layers largely unconstrained. We introduce Learnable Lens Networks (LLN), a physics-inspired architecture that replaces direct feature-space residual accumulation with learnable optical transport in an augmented position-angle phase space. Each layer alternates between free propagation, which provides an implicit transport path, and a learnable lens field that performs nonlinear trajectory transformation and focusing. Theoretically, we establish that LLN transport is globally invertible and volume-preserving for any differentiable lens field, with the implemented coordinate-wise Gaussian transport further satisfying symplecticity. Importantly, these structural constraints do not limit expressivity: with unrestricted embeddings and readouts, LLN retain universal approximation of continuous end-to-end maps. Experiments across diverse dynamical systems demonstrate that LLN improves long-horizon prediction while using substantially fewer parameters than same-depth comparators. Further analysis reveals stable depth-wise gradient transport and interpretable learned dynamics under the coupled propagation and refraction design.

\end{abstract}

\section{Introduction}

Depth enables neural networks to progressively construct abstract representations, but it also makes optimization increasingly sensitive to repeated compositions of layer transformations. Without appropriate architectural constraints, gradients may become vanishing, exploding, or poorly conditioned as signals propagate through many layers. Residual networks address this challenge by introducing an explicit identity shortcut:
\begin{equation}
x_{l+1}=x_l+f_l(x_l),
\label{eq:residual}
\end{equation}
which provides a direct pathway for information and gradient propagation \cite{he2016resnet}. In modern deep learning systems, however, residual connections are typically combined with additional stabilization mechanisms, including normalization, careful initialization, and residual scaling strategies \cite{ioffe2015batchnorm,ba2016layernorm,zhang2019fixup,bachlechner2021rezero}.

This observation highlights an important tension in deep residual learning. The identity branch improves backward signal transport, yet the additive feature-space formulation does not explicitly regulate how representations evolve across depth. When additional stabilization mechanisms are unavailable, repeated additive transformations may introduce uncontrolled changes in feature scale and geometry, even when gradients remain measurable. LLN follows the broader goal of stabilizing deep evolution, but takes a different route: rather than constraining an algebraic update externally, it changes the state space in which the transformation evolves.

Inspired by geometric optics, Learnable Lens Networks (LLN) represent hidden states in an augmented position-angle phase space. Each layer consists of free propagation followed by learnable refraction:
\begin{equation}
\begin{aligned}
y_{l+1} = y_l+L\theta_l,\
\theta_{l+1} = \theta_l+\phi_l(y_{l+1}),
\end{aligned}
\label{eq:lln_transport_intro}
\end{equation}
where \(y_l\) denotes the feature position, \(\theta_l\) is the angle state, \(L\) is the propagation distance, and \(\phi_l\) is a learnable lens field. Free propagation updates the position according to the current angle state, while the lens field modifies the angle according to the propagated position. This design replaces direct feature-space residual accumulation with structured phase-space transport. The resulting dynamics preserve an implicit transport path while allowing nonlinear transformations through learned refraction. Unlike conventional residual updates that operate within a single feature space, LLN separates transported states into complementary variables and enables geometric control over their evolution.

Such structured transport is particularly relevant when learned transformations are repeatedly composed. In long-horizon dynamical prediction, a one-step transition model is recursively applied to generate future states, causing small inaccuracies in representation evolution and transition dynamics to accumulate over time. Related representation-learning approaches therefore impose structure on latent evolution to support multistep prediction \cite{lusch2018deep}. This setting provides a natural testbed for evaluating whether a deep architecture can maintain stable transformations beyond single-step prediction.

The contributions of this work are summarized as follows:
\begin{itemize}
\item We introduce Learnable Lens Networks (LLN), a phase-space architecture for deep multilayer perceptrons (MLPs) that replaces direct feature-space residual accumulation with a learnable optical transport process over coupled position-angle states.

\item LLN admits an exact theoretical characterization: each internal transport layer is globally invertible and volume preserving for any differentiable lens field; the implemented Gaussian transport is symplectic, while a shallow LLN remains a universal approximator through its unrestricted embedding and readout.

\item LLN is evaluated through three empirical studies: fixed-depth long-horizon rollout performance, matched gradient-transport diagnostics, and a low-dimensional optical trajectory case study, with controlled component ablations supporting the transport mechanism. Together, these results examine predictive accuracy, parameter-count efficiency, depth-wise signal propagation, and the interpretability of learned phase-space transport.
\end{itemize}

\section{Related Work}

\paragraph{Deep-network stabilization.}
Residual connections made very deep networks practical by adding identity pathways for information and gradients \cite{he2016resnet}. Their modern use still depends on stabilization choices: BatchNorm, LayerNorm, pre-normalization, Admin, T-Fixup, ReZero, LayerScale, DeepNet, and depthwise transfer regulate activation scale, residual-branch magnitude, initialization, or depth-dependent tuning \cite{ioffe2015batchnorm,ba2016layernorm,xiong2020layernorm,liu2020admin,huang2020tfixup,bachlechner2021rezero,touvron2021going,wang2024deepnet,bordelon2024depthwise}. Normalization-free or normalization-alternative designs make the same point through scaled residual blocks, dynamic tanh transformations, and hyperspherical representations \cite{zhang2019fixup,brock2021normalizerfree,zhu2025transformers,loshchilov2025ngpt}. LLN pursues stable depth by replacing explicit feature-space skip addition with structured position-angle transport.

\paragraph{Dynamical and structure-preserving networks.}
Viewing depth as dynamical evolution has produced continuous-depth, augmented-state, antisymmetric, Hamiltonian, and stable residual networks \cite{chen2018neuralode,dupont2019augmented,chang2018antisymmetric,greydanus2019hamiltonian,haber2017stable}. Structure-preserving variants encode Hamiltonian, symplectic, or volume-preserving biases for learned physical dynamics \cite{zhong2020symplectic,jin2020sympnets,zhu2022vpnets}, while reversible networks and flows use invertibility mainly for memory efficiency or density modeling \cite{gomez2017revnet,behrmann2019invertible,dinh2017realnvp}. LLN instead constrains the repeated latent transport stack; its embedding and readout remain unrestricted, so dissipative observed-state maps are still expressible.

\paragraph{Neural operators and adaptive basis models.}
Neural operators learn families of differential-equation solution maps, with Fourier Neural Operators, physics-informed neural operators, and recent surveys emphasizing discretization transfer, partial differential equation (PDE) constraints, and simulation acceleration \cite{li2021fno,li2024pino,azizzadenesheli2024neural}. Recent work further studies sensitivity constraints, discretization mismatch, and Kolmogorov-Arnold network (KAN) style operator architectures \cite{behroozi2025sensitivity,gao2025discretization,lee2026kano}. Adaptive-basis models, including radial-basis-function networks and Kolmogorov-Arnold Networks, replace fixed pointwise nonlinearities with learnable basis expansions \cite{park1991rbf,liu2024kan,rigas2026initialization,li2025generalization}. LLN shares the scientific-dynamics setting and uses Gaussian basis learning, but it is evaluated as a compact fixed-resolution autoregressive predictor whose basis field refracts an angle variable inside a coupled transport map rather than serving as a standalone layer or resolution-independent operator kernel.

\section{Learnable Lens Networks}
\label{sec:method}

\begin{figure}[!hbp]
\centering
\includegraphics[width=\textwidth]{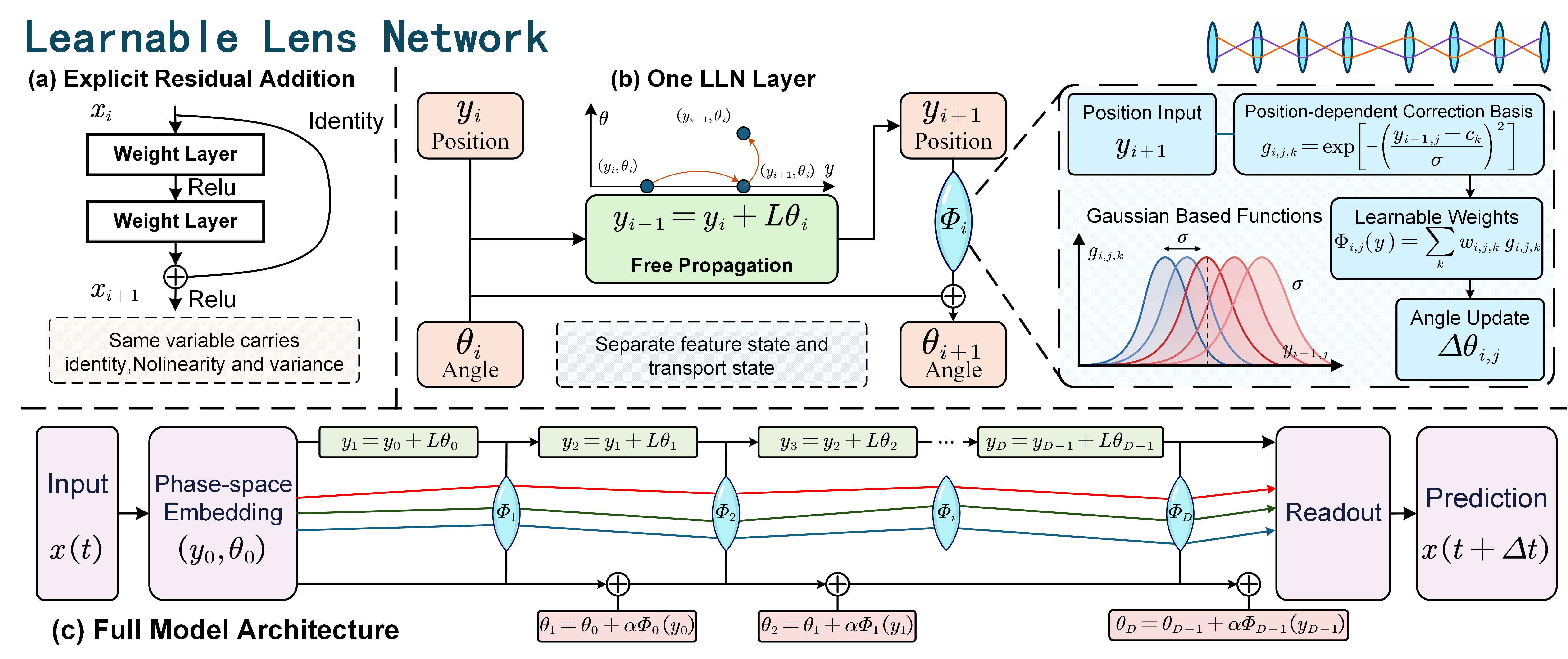}
\caption{\textbf{LLN architecture overview.}  (a) A conventional residual block adds a nonlinear transform in the same feature space, which can accumulate forward variance when normalization is removed.  (b) An LLN layer separates feature position $y$ and angle $\theta$; free propagation provides an implicit physical residual path, and a learnable lens field refracts the angle.  (c) Depth repeats the propagate-refract optical map, yielding a phase-space layer Jacobian with $\det(J_l)=1$.  Dashed connectors indicate how the single-layer operations in (b) are repeated across the stack in (c).}
\label{fig:architecture}
\end{figure}

\subsection{Phase-space representation}

Given an input $x\in\mathbb{R}^{d_{\mathrm{in}}}$, LLN maps the input into an augmented phase space with two coupled states:
\begin{align}
  y_0 &= W_yx+b_y,\\
  \theta_0 &= \tanh(W_\theta x+b_\theta),
\end{align}
where $y_0,\theta_0\in\mathbb{R}^{d}$. The variable $y_0$ represents the current feature position, while $\theta_0$ is an auxiliary directional state that controls how information is transported through depth. The bounded initialization of $\theta_0$ prevents excessively large initial displacements and provides a natural scale constraint when normalization layers are absent.

This phase-space representation is general for vector-valued inputs. In our experiments, low-dimensional chaotic systems are embedded with $d_{\mathrm{in}}=3$, while spatial and mechanical systems use input dimensions determined by the discretization resolution or number of masses, as described in the supplementary experimental details.

\subsection{Optical transport layer}

Each LLN layer consists of free propagation followed by learnable refraction:
\begin{align}
  y_{l+1} &= y_l+L\theta_l,
  \label{eq:prop}\\
  \theta_{l+1} &= \theta_l+\phi_l(y_{l+1}).
  \label{eq:refract}
\end{align}
The transport distance $L$ controls the position update, while the nonlinear field $\phi_l$ changes the directional state at the propagated position.

The lens field $\phi_l$ is parameterized by a coordinate-wise Gaussian basis expansion:
\begin{equation}
  \phi_l(y)_j=
  \sum_{k=1}^{K_{\mathrm{lens}}}
  a_{l,j,k}
  \exp\left[-\left(\frac{y_j-c_k}{\sigma}\right)^2\right],
  \label{eq:lens}
\end{equation}
where $a_{l,j,k}$ are learnable coefficients, $c_k$ are fixed control points, and $\sigma$ controls the locality of each basis function. This formulation provides a flexible and interpretable representation of the learned refraction field, allowing different feature coordinates to undergo localized trajectory transformations.

\begin{algorithm}[!htbp]
\caption{One-step LLN forward map}
\label{alg:lln_forward}
\begin{algorithmic}[1]
\Require state $x$; depth $D$; optical horizon $T$; lens fields $\{\phi_l\}_{l=0}^{D-1}$
\Ensure one-step prediction $\hat{x}_{+}$
\State Set the per-layer propagation length $L \gets T/D$.
\State Embed the position: $y \gets W_yx+b_y$.
\State Initialize the angle: $\theta \gets \tanh(W_\theta x+b_\theta)$.
\For{$l=0,\ldots,D-1$}
  \State Propagate the position: $y \gets y+L\theta$.
  \State Refract the angle: $\theta \gets \theta+\phi_l(y)$.
\EndFor
\State Read out the next state: $\hat{x}_{+}\gets W_oy+b_o$.
\State \Return $\hat{x}_{+}$.
\end{algorithmic}
\end{algorithm}
\vspace{-0.9em}

After $D$ transport layers, the final prediction is produced from the
position component as $\hat{x}^{+}=W_o y_D+b_o$. Because the readout
depends only on $y_D$, the terminal refraction updates the auxiliary angle
state $\theta_D$ without changing the one-step prediction. We retain this
update so that all $D$ stages share the same propagate-refract form and
the latent transport remains a uniform composition of square phase-space
maps. All reported parameter and arithmetic counts include this terminal
lens; the resulting efficiency ratios are therefore conservative relative
to an implementation that omits the readout-inactive update. The
directional state $\theta_D$ thus serves as an internal transport variable
rather than an output representation.

Algorithm~\ref{alg:lln_forward} assembles the embedding, optical transport, and readout into the complete one-step LLN forward map. We parameterize the per-layer distance as $L=T/D$, where $T$ is the total optical horizon. Increasing depth therefore introduces more learned refraction stages while keeping the nominal cumulative free-propagation distance fixed, separating architectural depth from the overall transport scale.

\subsection{Structured residual transport}

Unlike conventional residual blocks that combine identity transport and nonlinear transformation within a single feature variable, LLN separates these two roles into coupled phase-space states. The position variable $y_l$ stores the current representation, while $\theta_l$ determines its subsequent displacement. Consequently, information can propagate across depth through an explicit transport mechanism, while nonlinear transformations are introduced through the learned refraction field.

This separation provides a geometric interpretation of deep feature evolution. A sample can be viewed as a trajectory evolving through a sequence of learned optical transformations, where each layer modifies the trajectory rather than directly overwriting the representation. The resulting model remains a standard differentiable neural architecture while introducing a structured state transition mechanism.

\section{Geometric Analysis of LLN Dynamics}
\label{sec:geometric_analysis}

\subsection{Phase-space Jacobian}

We analyze one LLN layer through its Jacobian in the augmented phase space. Let $H_l:=\partial \phi_l(y_{l+1})/\partial y_{l+1}$ denote the Jacobian of the lens field. The propagation and refraction derivatives are
\begin{align}
\frac{\partial y_{l+1}}{\partial y_l} &= I,
& \frac{\partial y_{l+1}}{\partial \theta_l} &= LI,\nonumber\\
\frac{\partial \theta_{l+1}}{\partial y_l} &= H_l,
& \frac{\partial \theta_{l+1}}{\partial \theta_l} &= I+LH_l.
\end{align}

Therefore, the complete phase-space Jacobian of one LLN layer is

\begin{equation}
J_l=
\begin{bmatrix}
I & LI\\
H_l & I+LH_l
\end{bmatrix}.
\label{eq:jacobian}
\end{equation}

\begin{proposition}[Global invertibility and phase-space volume preservation]
\label{prop:lln_volume}
For any continuously differentiable lens field $\phi_l$, the LLN transport map
$F_l:(y_l,\theta_l)\mapsto(y_{l+1},\theta_{l+1})$ is a globally invertible
diffeomorphism with $\det(J_l)=1$, and hence preserves volume in the augmented phase space.
\end{proposition}

\begin{proof}
Applying the Schur complement with respect to the upper-left block $I$ gives
\begin{align}
\det(J_l)
&=\det(I)\det\!\left((I+LH_l)-H_lI^{-1}(LI)\right)\nonumber\\
&=\det(I+LH_l-LH_l)=1 .
\end{align}
Moreover, the inverse is explicit:
\begin{align}
\theta_l&=\theta_{l+1}-\phi_l(y_{l+1}),\qquad
y_l=y_{l+1}-L\theta_l .
\label{eq:lln_inverse}
\end{align}
It is globally defined and continuously differentiable whenever $\phi_l$ is. The change-of-variables formula then establishes global phase-space volume preservation.
\end{proof}

For the internal transport stack $\mathcal{T}_D=F_{D-1}\circ\cdots\circ F_0$, the multiplicative property of determinants gives
\begin{equation}
J_{\mathcal{T}_D}=J_{D-1}\cdots J_1J_0,
\qquad
\det(J_{\mathcal{T}_D})=\prod_{l=0}^{D-1}\det(J_l)=1 .
\end{equation}
This determinant statement applies to the repeated latent transport $\mathcal{T}_D$. Because the input embedding and output readout may change dimension, the end-to-end predictor retains the freedom to represent dissipative and other non-volume-preserving target maps. Appendix~\ref{sec:technical_details} proves the stronger symplectic property of the implemented coordinate-wise Gaussian transport and gives exact parameter-scaling results.

\paragraph{Implications.}
The volume-preserving property provides LLN with a geometric inductive bias that differs from unconstrained layer compositions. In particular, the internal transport cannot contract or expand phase-space volume, while the encoder and readout retain the flexibility required for general prediction tasks. Exact counting in Appendix~\ref{sec:lln_complexity} further shows that the depth-dependent learnable component scales as $O(DhK)$ for LLN, compared with $O(Dh^2)$ for a dense MLP of width $h$.

The determinant-level result leads directly to the empirical diagnostics. A unit determinant constrains the product $\prod_i\sigma_i(J_l)=1$, while individual singular values may evolve differently across depth. The implemented symplectic layer has reciprocal singular-value pairing, making gradient transport, activation statistics, and rollout stability natural empirical questions for Section~\ref{sec:experiments}.

This distinction separates LLN from architectures designed primarily for invertibility or likelihood estimation, such as normalizing flows. LLN does not use volume preservation as an objective for density modeling; rather, it incorporates phase-space geometry as an architectural bias for stable deep transformations and long-horizon dynamical prediction.

\section{Experiments}
\label{sec:experiments}

\subsection{Experimental setup}

We evaluate LLN on long-horizon autoregressive dynamics, where a learned one-step transition model is repeatedly composed and small errors can accumulate over time. Such settings test whether the underlying architecture supports stable forward evolution at increasing depths.

We use simulated trajectories rather than externally downloaded benchmark datasets. The systems cover Lorenz chaotic flows \cite{lorenz1963deterministic}, a synthetic mass-spring chain, and spatially discretized PDE surrogates based on Burgers dynamics \cite{burgers1948mathematical}, advection-diffusion dynamics \cite{hundsdorfer2003numerical}, and Allen-Cahn-type reaction-diffusion dynamics \cite{allen1979microscopic}. The training objective is the one-step transition map $F_{\Delta t}:x(t)\rightarrow x(t+\Delta t)$ generated by fourth-order Runge-Kutta simulation. During evaluation, the learned model is recursively applied as $\hat{x}_{k+1}=f_\theta(\hat{x}_k)$, and performance is measured by the normalized mean squared error (MSE) over long free-running rollouts.

All state variables are standardized using statistics computed from the training trajectories. We compare LLN with a plain ReLU MLP, residual MLP, scaled residual MLP with residual magnitude $1/\sqrt{D}$, ResMLP-BN, and LayerScale \cite{touvron2021going}. Unless otherwise specified, results are averaged over 10 random seeds. In addition to rollout error, we report parameter count and gradient-based diagnostics to analyze depth stability.


\begin{figure}[!htb]
\centering
\includegraphics[width=\columnwidth]{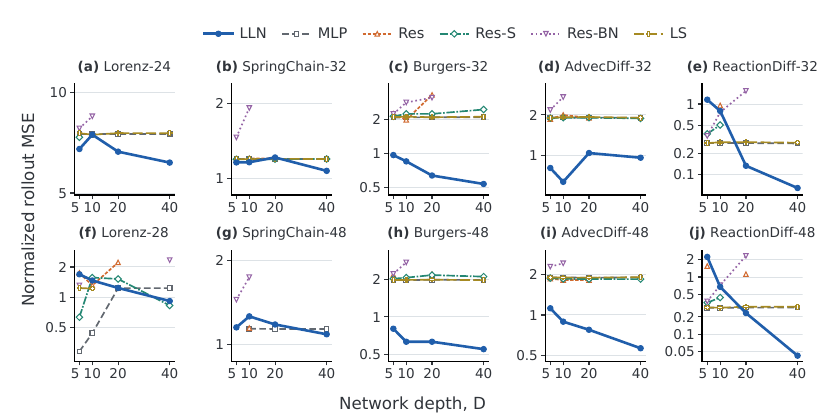}
\caption{\textbf{All-model rollout error across network depths.}
Curves show the mean normalized rollout MSE over ten runs for LLN and five non-LLN architectures at each evaluated depth. MLP denotes the plain multilayer perceptron, Res denotes the residual MLP, Res-S denotes the scaled residual MLP, Res-BN denotes the residual MLP with batch normalization (BN), and LS denotes LayerScale. Lower is better.}
\label{fig:physics_representative}
\end{figure}

\subsection{Experiment 1: Long-horizon rollout performance}

\begin{table*}[!h]
\centering
\caption{\textbf{Depth-40 selected-deployment rollout performance.}
Values denote normalized MSE (mean $\pm$ std over 10 seeds) over a 500-step free-running rollout.
The comparator is the lowest-MSE member of the designated non-LLN suite for each protocol block; all paired entries match depth, trajectory split, seed set, and evaluation metric.
The MSE and parameter ratios use the displayed deployments, whereas the floating-point operation (FLOP) ratio is a strict matched-width audit with both models set to $D=40$ and $h=128$.
Lower is better; bold marks the lower paired rollout MSE.}
\label{tab:physics_d40_conservative_main}
\resizebox{\textwidth}{!}{
\begin{tabular}{llcccccc}
\toprule
System & Family & LLN MSE & Selected non-LLN & Comparator MSE
& MSE ratio & Param ratio & FLOP ratio \\
\midrule
Lorenz-24
& chaotic flow
& $\mathbf{6.170\!\pm\!0.953}$
& Plain MLP
& $7.516\!\pm\!0.010$
& $0.821$ & $28.5\%$ & $29.2\%$ \\

Lorenz-28
& chaotic flow
& $0.920\!\pm\!0.545$
& Scaled ResMLP
& $\mathbf{0.831\!\pm\!0.942}$
& $1.106$ & $6.4\%$ & $29.0\%$ \\

SpringChain-32
& mechanical
& $\mathbf{1.070\!\pm\!0.379}$
& Plain MLP
& $1.193\!\pm\!0.001$
& $0.897$ & $38.9\%$ & $32.0\%$ \\

SpringChain-48
& mechanical
& $\mathbf{1.087\!\pm\!0.017}$
& Plain MLP
& $1.135\!\pm\!0.001$
& $0.958$ & $57.5\%$ & $33.4\%$ \\

Burgers-32
& fluid
& $\mathbf{0.537\!\pm\!0.155}$
& Plain MLP
& $2.096\!\pm\!0.020$
& $0.256$ & $5.5\%$ & $30.6\%$ \\

Burgers-48
& fluid
& $\mathbf{0.550\!\pm\!0.122}$
& LayerScale
& $1.973\!\pm\!0.030$
& $0.279$ & $5.9\%$ & $31.1\%$ \\

AdvecDiff-32
& transport
& $\mathbf{0.962\!\pm\!0.596}$
& Scaled ResMLP
& $1.883\!\pm\!0.064$
& $0.511$ & $22.1\%$ & $30.3\%$ \\

AdvecDiff-48
& transport
& $\mathbf{0.564\!\pm\!0.318}$
& Scaled ResMLP
& $1.854\!\pm\!0.033$
& $0.304$ & $23.8\%$ & $31.1\%$ \\

ReactionDiff-32
& reaction-diffusion
& $\mathbf{0.065\!\pm\!0.042}$
& Plain MLP
& $0.277\!\pm\!0.006$
& $0.235$ & $22.1\%$ & $30.6\%$ \\

ReactionDiff-48
& reaction-diffusion
& $\mathbf{0.042\!\pm\!0.031}$
& Plain MLP
& $0.291\!\pm\!0.005$
& $0.143$ & $23.8\%$ & $31.3\%$ \\
\bottomrule
\end{tabular}}
\end{table*}

We first examine whether LLN improves long-horizon prediction under a fixed depth of $D=40$. Table~\ref{tab:physics_d40_conservative_main} summarizes the rollout performance across ten system configurations from five dynamical families. It is a selected-deployment comparison assembled from three recorded protocol blocks. Seven original-grid rows use six-epoch Adam runs with method-specific recorded learning-rate settings; the two SpringChain rows use dedicated six-epoch Adam completion sweeps; and Lorenz-28 uses a shared 20-epoch AdamW recipe for both model families. Within every row, the simulated trajectory, temporal split, standardization, depth, seed set, and 500-step evaluation metric are identical. Appendix Tables~\ref{tab:experiment_protocol_summary} and~\ref{tab:main_lln_deployments} record the protocol index and selected LLN deployments.

The results reveal a consistent advantage of LLN on dissipative spatial systems. For Burgers dynamics, LLN reduces the rollout error to $25.6\%$ and $27.9\%$ of the selected comparator on the two resolutions, while achieving this performance with only $5.5\%$ and $5.9\%$ of the comparator parameters. Similar improvements are observed for advection-diffusion and reaction-diffusion systems, where LLN achieves lower rollout errors with substantially fewer parameters.

\begin{figure}[!htbp]
\centering
\includegraphics[width=1\textwidth]{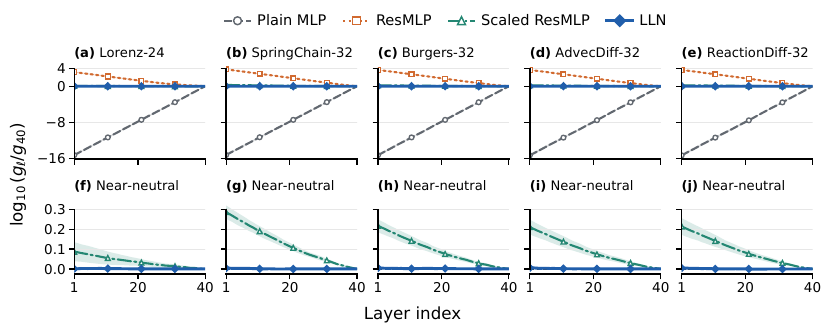}
\caption{\textbf{Matched layerwise gradient transport at $D=40$ and $h=128$.}
Curves show the mean $\log_{10}(g_l/g_{40})$ over 10 initialization
seeds on five representative systems; bands show one standard deviation.
The lower panels magnify Scaled ResMLP and LLN near neutral transport.}
\label{fig:gradient_transport_matched}
\end{figure}

We separately test initial-condition generalization beyond each main-table reference trajectory. Appendix Table~\ref{tab:multitrajectory_confirmation} freezes all ten selected $D=40$ LLN/comparator deployments before training on four perturbed trajectories and evaluating on four disjoint perturbed trajectories. Across six confirmatory training seeds, LLN has lower geometric-mean rollout error on seven of ten systems; five seed-clustered 95\% intervals remain fully below one (SpringChain-48, Burgers-32/48, AdvecDiff-48, and ReactionDiff-48). The Lorenz settings and SpringChain-32 favor the comparator under this stricter multi-trajectory protocol.

The advantage narrows on the mechanical systems, although LLN remains competitive with fewer parameters. Lorenz-28 is the one main-table row where LLN is slightly behind the selected comparator.

The matched arithmetic audit strengthens the efficiency result across every system in Table~\ref{tab:physics_d40_conservative_main}. With identical input/output dimensions, depth $D=40$, and width $h=128$, LLN uses only $29.0$-$33.4\%$ of the named baseline's simple arithmetic FLOPs, a reduction of $66.6$-$71.0\%$. Counts include dense multiply-adds, bias and ReLU operations, Gaussian-coordinate arithmetic, lens reduction, and residual/state updates.

The depth sweep in Figure~\ref{fig:physics_representative} reveals a consistent pattern. LLN improves with depth on both Burgers and reaction-diffusion resolutions, outperforming all reported baselines on reaction-diffusion at $D\geq20$. It also achieves the lowest mean error on both advection-diffusion settings at every evaluated depth. Across the remaining systems, LLN maintains competitive rollout performance while retaining substantial parameter and computational advantages. Overall, the results identify dissipative spatial dynamics as the regime in which LLN derives the most pronounced and consistent benefit from increased depth.

\subsection{Experiment 2: Matched depth-wise gradient transport}

The determinant and symplectic results characterize latent transport geometry; the remaining trainability question is how gradients move through a matched deep stack. Since prior work shows that deep-stack trainability can depend strongly on initialization, normalization placement, and residual-branch scale \cite{xiong2020layernorm,huang2020tfixup,liu2020admin,wang2024deepnet,bordelon2024depthwise}, we test gradient transport directly under a strictly matched architecture-level protocol. Plain MLP, ResMLP, Scaled ResMLP, and LLN all use depth $D=40$ and hidden width $h=128$; LLN uses $T=1$, $K=12$, and a learned initial angle. For each of ten initialization seeds, every model receives the same standardized data and four deterministic batches of 512 samples before any optimizer update.

Let $g_l=\operatorname{RMS}(\partial\mathcal{L}/\partial h_l)$ denote the root-mean-square (RMS) loss gradient of the state after repeated layer $l$, with $h_l=y_l$ for LLN. We report $\log_{10}(g_1/g_{40})$, which normalizes each model by its own final repeated-layer gradient scale and measures relative transport through the repeated stack. A value of zero means that the first and final repeated layers receive equal RMS gradients; negative and positive values indicate attenuation and amplification, respectively.

\begin{table}[!htbp]
\centering
\footnotesize
\caption{\textbf{Matched gradient transport at initialization.} Values are $\log_{10}(g_1/g_{40})$ (mean $\pm$ s.d. over 10 seeds). Zero is ideal; bold marks the closest value to zero.}
\label{tab:gradient_transport_matched}
\setlength{\tabcolsep}{3.8pt}
\begin{tabular*}{\textwidth}{@{\extracolsep{\fill}}lcccc@{}}
\toprule
System & Plain MLP & ResMLP & Scaled ResMLP & LLN \\
\midrule
Lorenz-24 & $-15.298\!\pm\!0.222$ & $3.140\!\pm\!0.152$ & $0.087\!\pm\!0.048$ & \textbf{$0.000\!\pm\!0.008$} \\
Lorenz-28 & $-15.298\!\pm\!0.219$ & $3.141\!\pm\!0.151$ & $0.088\!\pm\!0.048$ & \textbf{$0.001\!\pm\!0.008$} \\
SpringChain-32 & $-15.270\!\pm\!0.257$ & $3.741\!\pm\!0.064$ & $0.286\!\pm\!0.034$ & \textbf{$0.003\!\pm\!0.005$} \\
SpringChain-48 & $-15.306\!\pm\!0.214$ & $3.818\!\pm\!0.079$ & $0.333\!\pm\!0.034$ & \textbf{$0.003\!\pm\!0.008$} \\
Burgers-32 & $-15.326\!\pm\!0.265$ & $3.620\!\pm\!0.085$ & $0.217\!\pm\!0.032$ & \textbf{$0.004\!\pm\!0.006$} \\
Burgers-48 & $-15.355\!\pm\!0.207$ & $3.697\!\pm\!0.093$ & $0.251\!\pm\!0.043$ & \textbf{$0.004\!\pm\!0.007$} \\
AdvecDiff-32 & $-15.356\!\pm\!0.293$ & $3.624\!\pm\!0.083$ & $0.211\!\pm\!0.037$ & \textbf{$0.005\!\pm\!0.007$} \\
AdvecDiff-48 & $-15.297\!\pm\!0.125$ & $3.700\!\pm\!0.101$ & $0.248\!\pm\!0.032$ & \textbf{$0.005\!\pm\!0.005$} \\
ReactionDiff-32 & $-15.348\!\pm\!0.274$ & $3.624\!\pm\!0.088$ & $0.214\!\pm\!0.042$ & \textbf{$0.004\!\pm\!0.005$} \\
ReactionDiff-48 & $-15.338\!\pm\!0.300$ & $3.703\!\pm\!0.101$ & $0.251\!\pm\!0.035$ & \textbf{$0.008\!\pm\!0.004$} \\
\bottomrule
\end{tabular*}
\end{table}

\begin{figure}[!htbp]
\centering
\includegraphics[width=1\textwidth]{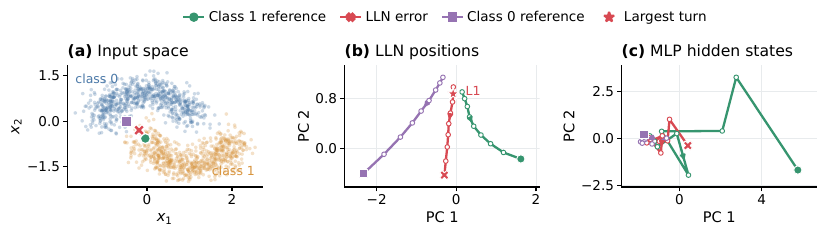}
\caption{\textbf{Low-dimensional optical case study.} (a) A deterministic boundary error and two nearest correct references. (b) LLN position paths. (c) Plain MLP paths for the same inputs.}
\label{fig:lln_error_case_study}
\end{figure}

Table~\ref{tab:gradient_transport_matched} shows a consistent architectural
separation. Plain MLP attenuates first-layer gradients by approximately
15 orders of magnitude, unscaled ResMLP amplifies them by more than three
orders, and residual scaling reduces but does not eliminate amplification.
LLN remains closest to neutral on every system, with mean log-ratios between
$0.000$ and $0.008$, corresponding to $g_1/g_{40}$ between $1.00$ and
$1.02$. Figure~\ref{fig:gradient_transport_matched} confirms the same
behavior across five representative systems: LLN's position-channel
gradients remain near the final-layer reference throughout depth. Together,
these results support near-neutral relative gradient transport at
initialization under the matched protocol.

\subsection{Experiment 3: Interpreting LLN through optical trajectories}

To expose LLN's internal position-angle evolution, we use two-moons classification under a frozen case-study protocol: data seed, model seed, noise level, width, depth, and training budget are fixed, and a deterministic boundary rule selects the least-confident LLN error plus the nearest correct sample from each relevant class.

Figures~\ref{fig:lln_error_case_study} and
\ref{fig:lln_error_refraction} trace the selected boundary error through
position trajectories, lens-induced turns, and the final readout. Each
principal component analysis (PCA) basis is fitted to all 400 held-out
representations, with its first two components explaining $97.0\%$ of LLN
position-state variance and $72.7\%$ of MLP hidden-state variance. The
highlighted layers are the misclassified sample's three largest turns under
$\Delta\alpha_l$. Across seven LLN errors, the median maximum turn is
$16.96^\circ$, versus $15.84^\circ$ for 393 correct samples; the displayed
case lies at the 55th percentile and therefore represents an ordinary
boundary decision rather than an extreme-turn outlier.

\begin{figure}[!htbp]
    \centering
    \includegraphics[width=1\textwidth]{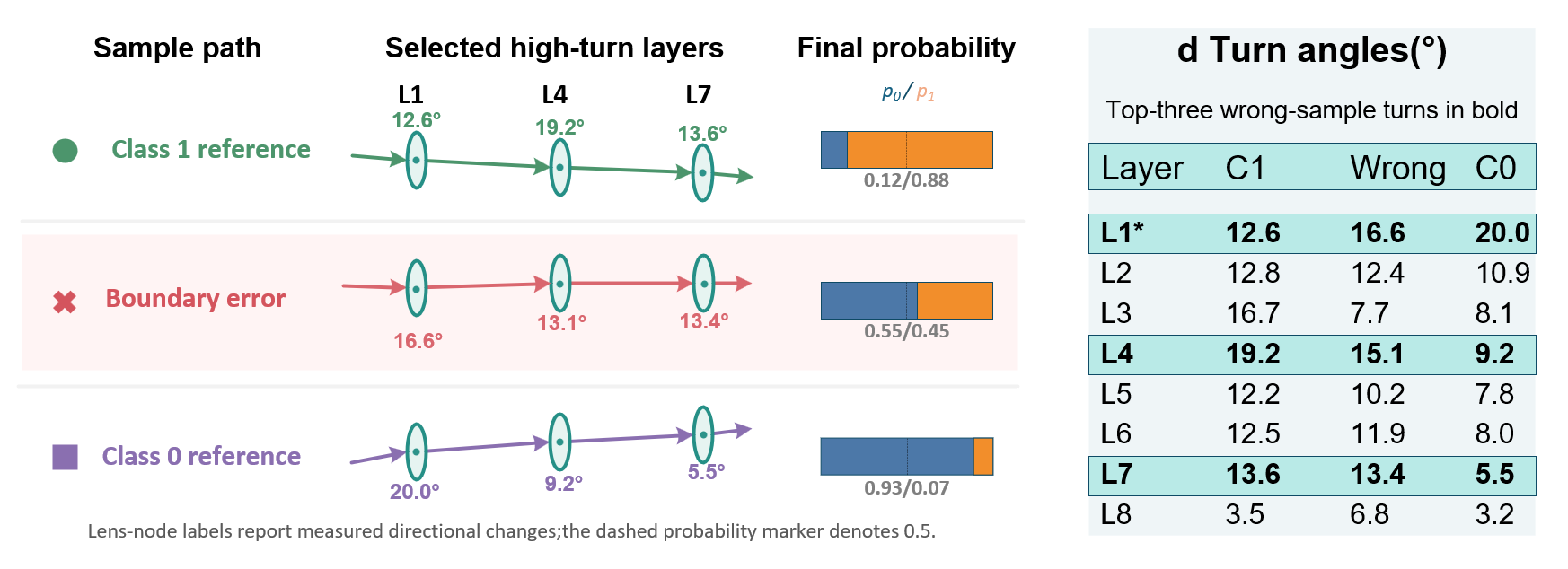}
    \caption{\textbf{Layerwise LLN refraction and readout.} Paths trace the class-1 reference, boundary error, and class-0 reference through the wrong sample's largest-turn layers (L1, L4, L7); the right table lists all-layer turns.}
\label{fig:lln_error_refraction}
\end{figure}

\section{Discussion and Conclusion}

LLN reframes depth as learned transport in an augmented phase space rather than repeated feature-space residual addition. Across the evaluated $D=40$ rollouts, it gives the clearest gains on dissipative spatial dynamics with fewer parameters than selected same-depth comparators. The all-system multi-trajectory confirmation preserves the strongest gains on five systems; the mechanism evidence shows near-neutral initialization gradients across ten systems, a best full structure in nine of ten controlled ablations, and traceable position-angle trajectories on two moons. Together, LLN emerges as a parameter-efficient phase-space alternative with analyzable latent geometry, compact depth-wise parameter scaling, empirical gradient neutrality, and inspectable learned trajectories.

\subsection*{Limitations and future work}

The benefits of LLN vary across dynamical regimes and deployment settings,
with the clearest improvements on the evaluated dissipative spatial systems.
Volume preservation and symplecticity constrain latent geometry but do not
guarantee favorable singular-value behavior, rollout accuracy, Lyapunov
stability, or end-to-end sensitivity. Evaluation is limited to simulated
systems, fixed-depth MLP-style one-step models, and a descriptive two-moons
case study. Future work should test broader architectures and develop
spectral and Lyapunov analyses beyond determinant-level geometry.

\subsection*{AI use statement}

Generative AI tools were used solely for auxiliary tasks, including language polishing, improving presentation clarity, refining figure captions, assisting with \LaTeX{} typesetting, and supporting preliminary literature searches. The core research idea, architectural design, mathematical formulations, theoretical results, theorem and proposition statements, and all associated derivations and proofs were independently developed and verified by the authors. Generative AI was not used to originate the proposed method, establish theoretical claims, design or implement the experiments, or draw conclusions from the results. All AI-assisted content was carefully reviewed, revised, and cross-checked by the authors, who take full responsibility for the accuracy, originality, and integrity of the final manuscript.
\subsection*{Ethics statement}

This study uses simulated and synthetic benchmarks, with no human subjects, personal data, or deployed decision-making systems.

\subsection*{Reproducibility statement}

Sections~\ref{sec:method}-\ref{sec:experiments} specify the model, theory, protocols, metrics, baselines, depth, and seed averaging; Appendices~\ref{sec:technical_details}-\ref{sec:supplementary_experimental_details} provide proofs, counts, ablations, configurations, and extended depth results. Code, seeds, raw tables, and figure scripts will be included in the supplement.

\bibliographystyle{iclr2027_conference}
\bibliography{iclr2027_conference}

\clearpage

\appendix
\raggedbottom
\section{Additional Theory and Expressivity}
\label{sec:technical_details}

This appendix completes the theoretical characterization used in the main text. We first show that the implemented coordinate-wise Gaussian transport is symplectic, derive the resulting reciprocal singular-value structure, and state precisely why this latent property does not constrain the observed dynamics to be Hamiltonian. We then give exact parameter and operation counts under the same convention used in the FLOP audit, followed by a constructive universal-approximation result. Appendix~\ref{sec:supplementary_experimental_details} collects the extended empirical evidence, and Appendix~\ref{sec:additional_optical_cases} contains additional optical trajectory cases.

\subsection{Symplectic structure of the implemented transport}

Let
\begin{equation}
\Omega=
\begin{bmatrix}
0&I\\
-I&0
\end{bmatrix}
\end{equation}
denote the canonical symplectic matrix on the augmented state $z=(y,\theta)$. A differentiable map with Jacobian $J$ is symplectic when $J^\top\Omega J=\Omega$.

\begin{corollary}[Symplectic Gaussian transport]
\label{cor:lln_symplectic}
If the lens Jacobian $H_l=\partial\phi_l/\partial y$ is symmetric, then one LLN transport layer is symplectic. Consequently, the coordinate-wise Gaussian lens in \eqref{eq:lens}, every finite composition of such layers, and their explicit inverse maps are symplectic.
\end{corollary}

\begin{proof}
The layer Jacobian in \eqref{eq:jacobian} factors as
\begin{equation}
J_l=
\underbrace{\begin{bmatrix}I&0\\H_l&I\end{bmatrix}}_{R_{H_l}}
\underbrace{\begin{bmatrix}I&LI\\0&I\end{bmatrix}}_{P_L}.
\end{equation}
The drift shear $P_L$ is symplectic because $LI$ is symmetric, and the refraction shear $R_{H_l}$ is symplectic whenever $H_l=H_l^\top$. Products and inverses of symplectic matrices are symplectic. For \eqref{eq:lens}, each output $\phi_l(y)_j$ depends only on $y_j$, so $H_l$ is diagonal and therefore symmetric.
\end{proof}

\begin{corollary}[Reciprocal singular-value pairing]
\label{cor:reciprocal_singular_values}
The singular values of the Jacobian of any implemented LLN transport stack occur in reciprocal pairs. In particular, its largest and smallest singular values satisfy $\sigma_{\max}=1/\sigma_{\min}$.
\end{corollary}

\begin{proof}
For a symplectic Jacobian $J$, the identity $J^\top\Omega J=\Omega$ gives $J^{-1}=-\Omega J^\top\Omega$. Since $\Omega$ is orthogonal, $J^{-1}$ and $J^\top$ have the same singular values. The singular values of $J^{-1}$ are the reciprocals of those of $J$, establishing the pairing.
\end{proof}

Corollary~\ref{cor:reciprocal_singular_values} is a structural statement, not a uniform conditioning bound: a symplectic Jacobian may simultaneously have very large and very small singular values. Moreover, the result concerns the square latent map $\mathcal{T}_D:(y_0,\theta_0)\mapsto(y_D,\theta_D)$. The input embedding and position-only readout are generally non-square and are not symplectic, so the end-to-end predictor is not restricted to Hamiltonian, symplectic, or volume-preserving dynamics in the observed state space.

\subsection{Parameter and operation scaling}
\label{sec:lln_complexity}

\begin{proposition}[Exact parameter counts]
\label{prop:parameter_counts}
Consider input dimension $n$, output dimension $m$, hidden width $h$, depth $D$, and $K$ fixed Gaussian control points per hidden coordinate. The implemented LLN with learned initial angle and the matched plain MLP contain, respectively,
\begin{align}
P_{\mathrm{LLN}}&=2h(n+1)+DhK+m(h+1),\label{eq:lln_parameter_count}\\
P_{\mathrm{MLP}}&=h(n+1)+Dh(h+1)+m(h+1).\label{eq:mlp_parameter_count}
\end{align}
For fixed $n,m,h,K$,
\begin{equation}
\lim_{D\rightarrow\infty}\frac{P_{\mathrm{LLN}}}{P_{\mathrm{MLP}}}
=\frac{K}{h+1}.
\label{eq:parameter_ratio_limit}
\end{equation}
Thus, at the matched audit setting $h=128$ and $K=12$, the depth-dominant parameter ratio is $12/129\approx9.3\%$.
\end{proposition}

\begin{proof}
LLN uses two affine input projections, contributing $2h(n+1)$ parameters; $D$ lens-weight arrays of shape $h\times K$, contributing $DhK$; and one affine readout, contributing $m(h+1)$. The control-point locations and Gaussian width are fixed buffers. The plain MLP uses one affine input projection, $D$ affine $h\times h$ hidden layers, and one affine readout, giving \eqref{eq:mlp_parameter_count}. Dividing both expressions by $D$ and taking the limit proves \eqref{eq:parameter_ratio_limit}. Among the audited residual variants, residual scaling introduces no additional parameter, LayerScale adds $Dh$ learned coefficients, and BatchNorm adds $2Dh$ affine parameters.
\end{proof}

For completeness, we also state the exact operation convention used to construct the FLOP columns in Tables~\ref{tab:physics_d40_conservative_main} and~\ref{tab:new_baselines_all_depths_detailed}. One multiply-add is counted as two arithmetic FLOPs; each ReLU is counted as one elementary operation; and exponential and hyperbolic-tangent evaluations are reported separately because their hardware costs are implementation dependent. Under this convention, one-sample LLN and plain-MLP forward passes require
\begin{align}
C_{\mathrm{LLN}}
&=2(2nh+hm)+(2h+m)+Dh(6K+3),\label{eq:lln_arithmetic_count}\\
C_{\mathrm{MLP}}
&=2(nh+Dh^2+hm)+2(D+1)h+m.\label{eq:mlp_arithmetic_count}
\end{align}
LLN additionally evaluates $DhK$ exponentials and $h$ hyperbolic tangents, whereas the MLP evaluates no transcendental functions. The LLN depth term combines $Dh(2K-1)$ operations for weighted lens reduction, $4DhK$ for basis coordinates, and $4Dh$ for state updates. Relative to \eqref{eq:mlp_arithmetic_count}, the audited unscaled ResMLP adds $Dh$ residual additions, the scaled ResMLP and LayerScale each add $2Dh$ multiply-add operations, and ResMLP-BN adds $5Dh$ normalization, affine, and residual operations. Hence the LLN-to-plain-MLP depth-dominant arithmetic ratio is
\begin{equation}
\lim_{D\rightarrow\infty}\frac{C_{\mathrm{LLN}}}{C_{\mathrm{MLP}}}
=\frac{6K+3}{2h+2},
\end{equation}
which is $75/258\approx29.1\%$ for $h=128$ and $K=12$. These counts establish parameter and arithmetic scaling; they do not by themselves imply a proportional wall-clock latency reduction, because kernel fusion, memory access, and transcendental throughput depend on the execution platform.

\subsection{Universal approximation}

We complement the geometric results with an expressivity statement. The theorem does not replace the conditioning analysis; it shows that constraining the internal transport does not prevent the end-to-end architecture from approximating general continuous maps on compact domains.

\begin{theorem}[Universal approximation by shallow LLN]
\label{thm:lln_ua}
Let $K\subset\mathbb{R}^{n}$ be compact, let $f:K\rightarrow\mathbb{R}^{q}$ be continuous, and let $\varepsilon>0$. Consider an LLN with nonzero propagation length $L$, a Gaussian lens dictionary containing at least one fixed control point $c$ with width $\sigma>0$ per coordinate,
\begin{equation}
  g_{c,\sigma}(s)=\exp\left[-\left(\frac{s-c}{\sigma}\right)^2\right],
\end{equation}
the input embedding from Section~\ref{sec:method}, two optical layers, and a linear readout from the final position channel. If the hidden dimension is allowed to be sufficiently large, then there exists a choice of LLN parameters such that
\begin{equation}
  \sup_{x\in K}\|\mathrm{LLN}(x)-f(x)\|_{\infty}<\varepsilon .
\end{equation}
\end{theorem}

\begin{proof}
Let $g(t)=\exp(-t^2)$. Since $g$ is continuous and non-polynomial, the classical universal approximation theorem implies that, for any $\varepsilon>0$, there exists a finite Gaussian ridge expansion
\begin{equation}
  F(x)=b+\sum_{i=1}^{M}\gamma_i g(w_i^\top x+\beta_i),
  \label{eq:ridge_expansion}
\end{equation}
where $b,\gamma_i\in\mathbb{R}^q$, with $\sup_{x\in K}\|F(x)-f(x)\|_\infty<\varepsilon$ \cite{cybenko1989approximation,hornik1991approximation,leshno1993multilayer,park1991rbf}. We only need to show that a shallow LLN can realize any finite expansion of this form.

Choose $h\geq 2M$. For each term $i$, allocate two coordinates $(p_i,r_i)$ and use the affine position embedding to set
\begin{align}
  y_{0,p_i}(x)=y_{0,r_i}(x)=c+\sigma(w_i^\top x+\beta_i),\qquad
  \theta_{0,p_i}(x)=\theta_{0,r_i}(x)=0 .
\end{align}
The zero angles are realized exactly by setting $W_\theta=0$ and $b_\theta=0$, since $\tanh(0)=0$. The first propagation therefore leaves these coordinates unchanged. Activate one Gaussian control point in the first lens on coordinate $p_i$ and set the corresponding lens on $r_i$ to zero:
\begin{equation}
  \phi_{0,p_i}(y)=g_{c,\sigma}(y_{p_i}),\qquad \phi_{0,r_i}(y)=0 .
\end{equation}
Then $\theta_{1,p_i}=g(w_i^\top x+\beta_i)$ and $\theta_{1,r_i}=0$. With the second lens set to zero, the second propagation gives
\begin{equation}
  y_{2,p_i}(x)-y_{2,r_i}(x)=L g(w_i^\top x+\beta_i).
\end{equation}
Finally, the linear readout assigns weights $\gamma_i/L$ and $-\gamma_i/L$ to $(p_i,r_i)$ and uses bias $b$, yielding exactly
\begin{equation}
  b+\sum_{i=1}^{M}\frac{\gamma_i}{L}\bigl(y_{2,p_i}(x)-y_{2,r_i}(x)\bigr)=F(x).
\end{equation}
Thus the constructed LLN approximates $f$ uniformly on $K$ to accuracy $\varepsilon$.
\end{proof}

\begin{remark}
The construction uses only a restricted subset of LLN parameters: paired coordinates, zero initial angles, one active Gaussian lens term, and a zero second lens. The practical architecture used in the experiments contains this construction as a special case. The theorem is existential: it establishes density in $C(K,\mathbb{R}^q)$ but does not provide a finite-width approximation rate, a conditioning guarantee, or an optimization guarantee. There is no conflict with latent volume preservation because the embedding and readout are unrestricted, generally non-square maps.
\end{remark}

\FloatBarrier
\section{Extended Empirical Evidence}
\label{sec:supplementary_experimental_details}

\subsection{Experiment Protocols and Selected Configurations}

All dynamics trajectories are generated from the source-code right-hand sides and RK4 integrator rather than loaded from external data files. For spatial systems, the suffix denotes the spatial resolution or system size. Specifically, Burgers-$n$, AdvecDiff-$n$, and ReactionDiff-$n$ contain $n$ discretization points, while SpringChain-$n$ contains $n$ masses with a state dimension of $2n$ due to positions and momenta. Lorenz systems always have three state variables; the suffix corresponds to the dynamical parameter rather than dimensionality. Thus Lorenz-24 and Lorenz-28 are parameter settings, not 24- and 28-dimensional datasets.

Appendix Table~\ref{tab:experiment_protocol_summary} indexes the empirical protocols, while Appendix Table~\ref{tab:main_lln_deployments} lists the selected LLN deployments for the main $D=40$ rollout table. This keeps result tables focused on their diagnostic messages instead of repeating hyperparameters. Rows marked as all-system cover the full ten-system set; figure-level and hyperparameter diagnostics that use subsets are explicitly labeled by their displayed systems and are not used as all-system claims.

The implementation additionally uses a fixed scalar lens scale $s$, which multiplies the learned lens correction before the angle update:
\begin{equation}
\theta_{l+1}=\theta_l+s\phi_l(y_{l+1}).
\label{eq:scaled_lens_update}
\end{equation}
In the main text, $s$ is absorbed into the definition of $\phi_l$ to keep the notation compact. The scale controls the magnitude of lens refraction without introducing additional learnable parameters and is set to $s=1$ unless otherwise listed. Setting $s=0$ disables the refraction branch while retaining its parameter tensors, yielding the \emph{w/o lens} ablation in Appendix Table~\ref{tab:lln_structure_ablation}.

Unless otherwise stated, dynamics rollouts use 10 random seeds, 3000 training pairs, 700 test pairs, and 500-step free-running evaluation. The original-grid recipe uses six Adam epochs, with LLN learning rate $10^{-3}$ and the tuned non-LLN recipe using learning rate $3{\times}10^{-4}$, weight decay $10^{-6}$, and gradient clipping 1.0. The Lorenz-28 completion protocol uses 20 AdamW epochs with learning rate $5{\times}10^{-4}$, weight decay $10^{-5}$, gradient clipping 1.0, and cosine decay.

\begin{table*}[!t]
\centering
\small
\caption{\textbf{Protocol index for the empirical evidence.}
System abbreviations are L=Lorenz, SC=SpringChain, B=Burgers,
AD=AdvecDiff, and RD=ReactionDiff.}
\label{tab:experiment_protocol_summary}
\setlength{\tabcolsep}{4.5pt}
\renewcommand{\arraystretch}{1.08}
\begin{tabular}{
@{}
L{0.20\textwidth}
L{0.21\textwidth}
L{0.39\textwidth}
L{0.10\textwidth}
@{}}
\toprule
Evidence block
& Systems
& Selection/configuration
& Protocol \\
\midrule

Table~\ref{tab:physics_d40_conservative_main}: main rollout
& L24/28; SC32/48; B32/48; AD32/48; RD32/48
& Selected LLN configuration versus the lowest-MSE designated
same-depth non-LLN comparator; LLN configurations are listed in
Table~\ref{tab:main_lln_deployments}.
& R0, SC-C, L28-C \\

\addlinespace[0.18em]

Table~\ref{tab:multitrajectory_confirmation}: multi-trajectory confirmation
& L24/28; SC32/48; B32/48; AD32/48; RD32/48
& Main-table LLN and comparator configurations fixed before
confirmatory retraining; reported estimates use training seeds 11-16.
& MT \\

\addlinespace[0.18em]

Figure~\ref{fig:experiment1_temporal_rollout}: rollout curves
& RD48, AD48, B48, L24
& Configurations selected using seeds 0-2 and evaluated using
disjoint seeds 3-9; $D=40$, $K=12$, $s=1$, and learned $\theta_0$.
This is a representative visualization rather than an all-system claim.
& C0 \\

\addlinespace[0.18em]

Table~\ref{tab:gradient_transport_matched}: gradient transport
& L24/28; SC32/48; B32/48; AD32/48; RD32/48
& All models use $D=40$ and $h=128$; LLN additionally uses
$T=1$, $K=12$, $s=1$, and learned $\theta_0$.
& G0 \\

\addlinespace[0.18em]

Table~\ref{tab:lln_structure_ablation}: structure ablation
& L24/28; SC32/48; B32/48; AD32/48; RD32/48
& Main-table LLN configurations; the three ablations set
$s=0$ (\emph{w/o lens}), $\theta_0=0$, or $L=0$
(\emph{w/o propagation}), respectively.
& R0 \\

\addlinespace[0.18em]

Table~\ref{tab:new_baselines_all_depths_detailed}: all depths
& L24/28; SC32/48; B32/48; AD32/48; RD32/48
& $D\in\{5,10,20,40\}$; baselines use $h=128$.
The LLN grid varies $T$ and $h$ with default $K=12$ and $s=1$;
completion rows use their recorded system-specific configurations.
& R0, SC-C, L28-C \\

\addlinespace[0.18em]

Architecture-selection factorial
& L28, SC32, B32, AD32
& Factorial over $D$, $h$, and $T$, with default $K=12$, $s=1$,
and learned $\theta_0$; used as a representative
deployment-selection diagnostic.
& AS \\

\addlinespace[0.18em]

Optical hyperparameter profiles
& L28, SC32, B32, AD32
& One-factor profiles over lens scale, basis count, and learned
versus zero $\theta_0$; used as a representative hyperparameter
diagnostic.
& HP \\

\addlinespace[0.18em]

Experiment~3: optical case study
& Two-moons classification
& Archived configuration fixed before analysis; displayed samples
are selected using deterministic confidence and nearest-reference rules.
& CS \\

\bottomrule
\end{tabular}

\vspace{0.25em}
\footnotesize
\emph{Protocol codes.}
R0: original-grid rollout;
SC-C: SpringChain completion recipe;
L28-C: Lorenz-28 completion recipe;
MT: disjoint train/test trajectories;
C0: rollout-curve confirmation;
G0: initialization gradient audit;
AS: architecture selection;
HP: optical hyperparameters;
CS: optical case study.
\end{table*}

\begin{table*}[!t]
\centering
\caption{\textbf{Selected LLN deployments for the main \(D=40\) rollout results.}
\(T\) denotes the optical horizon, \(h\) the hidden width, \(K\) the number of Gaussian basis functions, and \(s\) the fixed lens-output scale. All configurations use a learned initial angle \(\theta_0\).}
\label{tab:main_lln_deployments}
\small
\setlength{\tabcolsep}{8pt}
\renewcommand{\arraystretch}{1.08}
\begin{tabular*}{\textwidth}{@{\extracolsep{\fill}}lcccc@{}}
\toprule
System & \(T\) & \(h\) & \(K\) & \(s\) \\
\midrule
Lorenz-24        & \(3\)    & \(384\) & \(12\) & \(1\)     \\
Lorenz-28        & \(0.75\) & \(128\) & \(8\)  & \(0.075\) \\
SpringChain-32   & \(1\)    & \(512\) & \(8\)  & \(1.25\)  \\
SpringChain-48   & \(0.25\) & \(512\) & \(12\) & \(1\)     \\
Burgers-32       & \(2\)    & \(64\)  & \(12\) & \(1\)     \\
Burgers-48       & \(2\)    & \(64\)  & \(12\) & \(1\)     \\
AdvecDiff-32     & \(1\)    & \(256\) & \(12\) & \(1\)     \\
AdvecDiff-48     & \(1\)    & \(256\) & \(12\) & \(1\)     \\
ReactionDiff-32  & \(2\)    & \(256\) & \(12\) & \(1\)     \\
ReactionDiff-48  & \(1\)    & \(256\) & \(12\) & \(1\)     \\
\bottomrule
\end{tabular*}
\end{table*}

\subsection{Frozen Multi-Trajectory Confirmation}

The main table controls temporal splits and random initialization within each system's reference simulation. To test initial-condition generalization separately, we freeze all ten selected $D=40$ LLN and comparator deployments from Table~\ref{tab:physics_d40_conservative_main} and generate disjoint training and test trajectory sets. For each system, four training trajectories and four test trajectories use fixed, independently drawn perturbations with RMS scale 2\% of the reference initial state; their trajectory seeds are $\{201,202,203,204\}$ and $\{101,102,103,104\}$, respectively. Each training trajectory contributes 3000 one-step pairs. Every trained model is evaluated on two non-overlapping 500-step windows from each test trajectory.

Table~\ref{tab:multitrajectory_confirmation} reports geometric means across all trajectory windows after first averaging log ratios within each training seed. The result table filters to confirmatory seeds 11-16; earlier seed-10 pilot rows in the raw archive are excluded from every estimate. The 95\% intervals resample confirmatory seeds 11-16 as clusters, so the eight trajectory-window evaluations from one trained model are not treated as independent replicates. Configurations, trajectory seeds, windows, and analysis code are fixed for these six seeds, which are also disjoint from the earlier selection and confirmation seeds.

\begin{center}
\begin{minipage}{\textwidth}
\centering
\small
\captionof{table}{\textbf{Frozen multi-trajectory confirmation across all ten $D=40$ systems.} MSE values are geometric means over six confirmatory training seeds, four unseen initial-condition trajectories, and two non-overlapping 500-step windows per trajectory. Ratio intervals are 95\% seed-cluster bootstrap intervals. Lower is better; bold marks the lower paired MSE, and bold ratios mark intervals entirely below one.}
\label{tab:multitrajectory_confirmation}
\setlength{\tabcolsep}{4.5pt}
\renewcommand{\arraystretch}{1.08}
\begin{tabular*}{\textwidth}{@{\extracolsep{\fill}}lccc@{}}
\toprule
System & LLN MSE & Comparator MSE & MSE ratio [95\% CI] \\
\midrule
Lorenz-24 & $2.027$ & \textbf{$0.554$} & $3.656\ [1.180,13.444]$ \\
Lorenz-28 & $2.760$ & \textbf{$0.448$} & $6.165\ [4.769,7.947]$ \\
SpringChain-32 & $1.919$ & \textbf{$0.935$} & $2.051\ [1.560,2.741]$ \\
SpringChain-48 & \textbf{$0.053$} & $0.950$ & $\mathbf{0.056}\ [0.037,0.083]$ \\
Burgers-32 & \textbf{$0.289$} & $0.999$ & $\mathbf{0.289}\ [0.177,0.482]$ \\
Burgers-48 & \textbf{$0.199$} & $4.922$ & $\mathbf{0.040}\ [0.00109,0.451]$ \\
AdvecDiff-32 & \textbf{$0.344$} & $120.233$ & $0.00286\ [9.44e-08,1.006]$ \\
AdvecDiff-48 & \textbf{$0.164$} & $28.028$ & $\mathbf{0.00585}\ [0.000135,0.142]$ \\
ReactionDiff-32 & \textbf{$0.914$} & $1.227$ & $0.745\ [0.454,1.066]$ \\
ReactionDiff-48 & \textbf{$0.569$} & $0.935$ & $\mathbf{0.608}\ [0.501,0.735]$ \\
\bottomrule
\end{tabular*}
\end{minipage}
\end{center}

Five seed-cluster intervals remain entirely below one: SpringChain-48, Burgers-32, Burgers-48, AdvecDiff-48, and ReactionDiff-48. AdvecDiff-32 and ReactionDiff-32 have LLN-favorable geometric means but intervals crossing one, while Lorenz-24, Lorenz-28, and SpringChain-32 favor the comparator. The confirmation therefore supports robust initial-condition generalization on a subset of systems but rules out a blanket all-system claim.

\subsection{Supplementary Temporal Rollout Curves}

Figure~\ref{fig:experiment1_temporal_rollout} gives a step-wise view of the main rollout behavior. The plotted configurations are selected using seeds $\{0,1,2\}$ only, retrained with disjoint confirmation seeds $\{3,\ldots,9\}$, and evaluated for 500 free-running steps. The three dissipative PDE representatives match the strongest pattern in Table~\ref{tab:physics_d40_conservative_main}: LLN remains below the selected baseline at every displayed step on ReactionDiff-48, AdvecDiff-48, and Burgers-48. Lorenz-24 is included as a boundary case and shows mixed chaotic-flow behavior.

\begin{center}
\begin{minipage}{\textwidth}
\centering
\includegraphics[width=\textwidth]{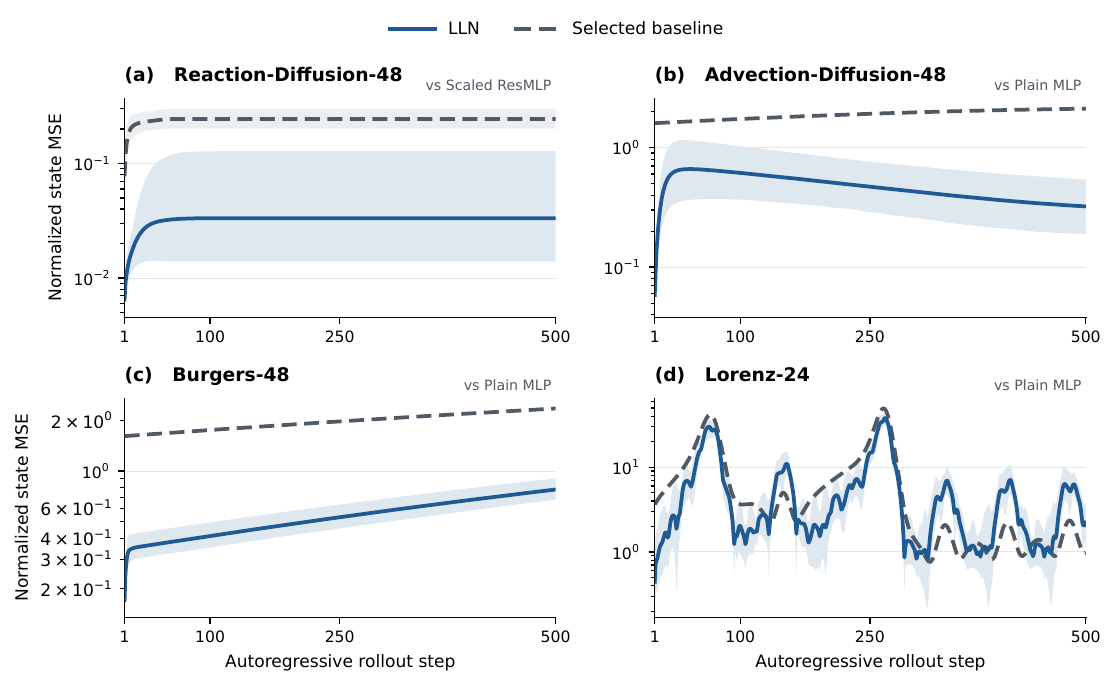}
\captionof{figure}{\textbf{Temporal rollout error at depth $D=40$.} Lines show geometric-mean normalized state MSE over seven confirmation seeds; bands are 95\% seed-bootstrap intervals. Lower is better.}
\label{fig:experiment1_temporal_rollout}
\end{minipage}
\end{center}

\subsection{Controlled Structure Ablation}

We isolate the three mechanisms that define LLN's propagate-refract update across all ten selected $D=40$ deployments: free propagation, lens refraction, and the learned initial angle. The full deployments are those in Table~\ref{tab:main_lln_deployments}; each uses a learned initial angle. Each ablation changes exactly one mechanism by setting the lens scale to zero, setting $\theta_0=0$, or setting the propagation step $L=0$. The model tensors, depth, optimizer, data, and random seeds otherwise remain unchanged.

\begin{table*}[t]
\centering
\small
\caption{\textbf{Complete LLN structure ablation across all ten systems at $D=40$.}
Values are 500-step normalized rollout MSE (mean $\pm$ s.d. over 10 independent stochastic runs). Extremely large divergent values are reported in scientific notation. Within each system, all variants use the same width, optical horizon, Gaussian basis count, optimizer, and data protocol; each ablation modifies only the indicated architectural component. Lower is better, and bold marks the best value in each row.}
\label{tab:lln_structure_ablation}
\setlength{\tabcolsep}{5.5pt}
\renewcommand{\arraystretch}{1.08}
\begin{tabular*}{\textwidth}{@{\extracolsep{\fill}}lcccc@{}}
\toprule
System
& Full LLN
& Without lens refraction
& Zero initial angle
& Without free propagation \\
\midrule
Lorenz-24
& \textbf{$6.18\!\pm\!0.98$}
& $8.54\!\pm\!3.94$
& $9.95\!\pm\!3.43$
& $\infty$ \\

Lorenz-28
& \textbf{$1.03\!\pm\!0.48$}
& $1.37\!\pm\!0.66$
& $5.6{\times}10^{10}$
& $1.0{\times}10^{5}$ \\

SpringChain-32
& $1.22\!\pm\!0.18$
& $1.22\!\pm\!0.13$
& $1.19\!\pm\!0.01$
& \textbf{$1.17\!\pm\!0.01$} \\

SpringChain-48
& \textbf{$1.09\!\pm\!0.02$}
& $1.10\!\pm\!0.01$
& $1.11\!\pm\!0.01$
& $1.11\!\pm\!0.00$ \\

Burgers-32
& \textbf{$0.54\!\pm\!0.16$}
& $0.79\!\pm\!0.26$
& $1.78\!\pm\!1.00$
& $82.25\!\pm\!253.94$ \\

Burgers-48
& \textbf{$0.55\!\pm\!0.12$}
& $0.79\!\pm\!0.27$
& $1.22\!\pm\!0.55$
& $1.9{\times}10^{16}$ \\

AdvecDiff-32
& \textbf{$0.95\!\pm\!0.57$}
& $3.61\!\pm\!2.27$
& $1.1{\times}10^{25}$
& $\infty$ \\

AdvecDiff-48
& \textbf{$0.56\!\pm\!0.31$}
& $1.32\!\pm\!0.88$
& $1.5{\times}10^{33}$
& $\infty$ \\

ReactionDiff-32
& \textbf{$0.064\!\pm\!0.041$}
& $2.66\!\pm\!2.23$
& $0.16\!\pm\!0.09$
& $6.2{\times}10^{33}$ \\

ReactionDiff-48
& \textbf{$0.040\!\pm\!0.033$}
& $1.85\!\pm\!0.75$
& $0.091\!\pm\!0.017$
& $0.28\!\pm\!0.04$ \\
\bottomrule
\end{tabular*}
\end{table*}

The complete LLN is best on nine of ten systems; the exception is SpringChain-32, where all variants are close and no-propagation is slightly lower. Removing lens refraction raises rollout error on every system, with small changes on the spring chains but large penalties on the PDE systems. Zeroing the initial angle worsens nine of ten systems, again excepting SpringChain-32. Setting $L=0$ creates severe divergence or orders-of-magnitude degradation on Lorenz and PDE systems, while the two spring-chain rows remain near the full model. These interventions support the coupled transport mechanism as important for most deployments, especially chaotic and dissipative systems, while identifying a mechanical boundary case rather than forcing a universal claim.

\subsection{All-Depth and Extended-Baseline Diagnostics}

Table~\ref{tab:new_baselines_all_depths_detailed} reports the fixed-depth rollout comparison over $D\in\{5,10,20,40\}$ on the ten-system set, using ten seeds per configuration. The original eight-system depth sweep is extended with SpringChain-48 and Lorenz-28 completion rows; the $D=40$ SpringChain-32 and SpringChain-48 LLN entries match the retuned deployments used in Table~\ref{tab:physics_d40_conservative_main}. For each row, the LLN and non-LLN entries are selected from the designated archived sweep at that depth, and the last three columns report LLN-to-comparator MSE, parameter, and strictly matched arithmetic-FLOP ratios. Non-finite or clearly divergent rollout values are shown as $\infty$. For the FLOP audit, LLN and its comparator use the same system dimension, depth, and hidden width $h=128$; LLN's $1536D$ exponential evaluations and 128 hyperbolic tangents per forward pass are reported separately from simple arithmetic FLOPs. The table expands the main depth-40 summary and makes LLN's strengthening margin at larger depths on dissipative PDEs explicit; LayerScale is included under the same tuned-baseline protocol.

\begin{table*}[!t]
\centering
\scriptsize
\setlength{\tabcolsep}{1.2pt}
\renewcommand{\arraystretch}{1.16}
\caption{\textbf{All-depth rollout comparison on the ten-system benchmark.}
Values are normalized rollout MSE (mean $\pm$ s.d. over 10 seeds);
the final columns report LLN-to-comparator MSE, parameter, and matched
arithmetic-FLOP ratios. LS denotes LayerScale and ``r.'' denotes ratio.
Lower is better; bold marks the lowest finite rollout MSE among the
compared models in each row.}
\label{tab:new_baselines_all_depths_detailed}
\resizebox{\textwidth}{!}{%
\begin{tabular}{@{}clccccccccc@{}}
\toprule
$D$ & System & LLN & MLP & Res & Res-S & Res-BN & LS
& MSE r. & Param r. & FLOP r. \\
\midrule

5 & Lorenz-24
& \textbf{$6.778\!\pm\!2.160$}
& 7.510 $\pm$ 0.066
& $\infty$
& 7.348 $\pm$ 0.805
& 7.813 $\pm$ 1.286
& 7.562 $\pm$ 0.294
& 0.922 & 0.054 & 0.301 \\

5 & Lorenz-28
& 1.692 $\pm$ 1.300
& \textbf{$0.292\!\pm\!0.165$}
& 1.742 $\pm$ 0.907
& 0.637 $\pm$ 0.862
& 1.313 $\pm$ 1.089
& 1.236 $\pm$ 0.332
& 5.790 & 0.078 & 0.303 \\

5 & SpringChain-32
& \textbf{$1.157\!\pm\!0.033$}
& 1.194 $\pm$ 0.003
& 1.187 $\pm$ 0.018
& 1.191 $\pm$ 0.009
& 1.461 $\pm$ 0.065
& 1.193 $\pm$ 0.006
& 0.975 & 0.985 & 0.490 \\

5 & SpringChain-48
& \textbf{$1.151\!\pm\!0.112$}
& $\infty$
& $\infty$
& $\infty$
& 1.448 $\pm$ 0.201
& $\infty$
& 0.794 & 0.826 & 0.560 \\

5 & Burgers-32
& \textbf{$0.969\!\pm\!0.319$}
& 2.090 $\pm$ 0.024
& 2.250 $\pm$ 0.127
& 2.138 $\pm$ 0.048
& 2.244 $\pm$ 0.270
& 2.105 $\pm$ 0.037
& 0.464 & 0.112 & 0.401 \\

5 & Burgers-48
& \textbf{$0.803\!\pm\!0.279$}
& 1.984 $\pm$ 0.016
& 2.185 $\pm$ 0.170
& 2.026 $\pm$ 0.051
& 2.218 $\pm$ 0.251
& 1.977 $\pm$ 0.030
& 0.406 & 0.138 & 0.445 \\

5 & AdvecDiff-32
& \textbf{$0.807\!\pm\!0.339$}
& 1.913 $\pm$ 0.020
& 1.863 $\pm$ 0.084
& 1.895 $\pm$ 0.044
& 2.173 $\pm$ 0.297
& 1.903 $\pm$ 0.041
& 0.433 & 0.668 & 0.399 \\

5 & AdvecDiff-48
& \textbf{$1.122\!\pm\!0.825$}
& 1.905 $\pm$ 0.014
& 1.868 $\pm$ 0.035
& 1.884 $\pm$ 0.027
& 2.296 $\pm$ 0.168
& 1.907 $\pm$ 0.025
& 0.601 & 0.833 & 0.447 \\

5 & ReactionDiff-32
& 1.150 $\pm$ 0.609
& \textbf{$0.279\!\pm\!0.009$}
& $\infty$
& 0.378 $\pm$ 0.064
& 0.353 $\pm$ 0.066
& 0.280 $\pm$ 0.027
& 4.123 & 0.668 & 0.401 \\

5 & ReactionDiff-48
& 2.244 $\pm$ 0.950
& \textbf{$0.289\!\pm\!0.008$}
& 1.546 $\pm$ 3.251
& 0.355 $\pm$ 0.035
& 0.374 $\pm$ 0.036
& 0.293 $\pm$ 0.016
& 7.775 & 0.139 & 0.448 \\

\midrule

10 & Lorenz-24
& 7.482 $\pm$ 1.815
& 7.511 $\pm$ 0.018
& $\infty$
& $\infty$
& 8.496 $\pm$ 1.854
& \textbf{$7.479\!\pm\!0.286$}
& 1.000 & 0.050 & 0.295 \\

10 & Lorenz-28
& 1.470 $\pm$ 1.091
& \textbf{$0.442\!\pm\!0.316$}
& 1.312 $\pm$ 1.068
& 1.570 $\pm$ 1.390
& $\infty$
& 1.223 $\pm$ 0.571
& 3.324 & 0.070 & 0.297 \\

10 & SpringChain-32
& \textbf{$1.159\!\pm\!0.064$}
& 1.193 $\pm$ 0.002
& 1.201 $\pm$ 0.014
& 1.190 $\pm$ 0.007
& 1.918 $\pm$ 0.206
& 1.193 $\pm$ 0.004
& 0.974 & 0.664 & 0.398 \\

10 & SpringChain-48
& 1.261 $\pm$ 0.244
& \textbf{$1.138\!\pm\!0.010$}
& 1.145 $\pm$ 0.019
& $\infty$
& 1.744 $\pm$ 0.251
& $\infty$
& 1.109 & 0.553 & 0.448 \\

10 & Burgers-32
& \textbf{$0.852\!\pm\!0.313$}
& 2.091 $\pm$ 0.013
& 1.992 $\pm$ 0.983
& 2.246 $\pm$ 0.070
& 2.836 $\pm$ 0.555
& 2.120 $\pm$ 0.029
& 0.428 & 0.081 & 0.347 \\

10 & Burgers-48
& \textbf{$0.631\!\pm\!0.174$}
& 1.974 $\pm$ 0.008
& $\infty$
& 2.067 $\pm$ 0.083
& 2.743 $\pm$ 0.541
& 1.981 $\pm$ 0.032
& 0.320 & 0.096 & 0.375 \\

10 & AdvecDiff-32
& \textbf{$0.637\!\pm\!0.226$}
& 1.907 $\pm$ 0.015
& 1.996 $\pm$ 0.228
& 1.907 $\pm$ 0.054
& 2.713 $\pm$ 0.449
& 1.922 $\pm$ 0.030
& 0.334 & 0.483 & 0.346 \\

10 & AdvecDiff-48
& \textbf{$0.891\!\pm\!0.558$}
& 1.907 $\pm$ 0.005
& 1.826 $\pm$ 0.224
& 1.862 $\pm$ 0.030
& 2.434 $\pm$ 0.391
& 1.904 $\pm$ 0.016
& 0.488 & 0.384 & 0.374 \\

10 & ReactionDiff-32
& 0.797 $\pm$ 0.535
& \textbf{$0.282\!\pm\!0.006$}
& 0.952 $\pm$ 1.387
& 0.505 $\pm$ 0.077
& 0.762 $\pm$ 0.471
& 0.288 $\pm$ 0.015
& 2.830 & 0.483 & 0.348 \\

10 & ReactionDiff-48
& 0.672 $\pm$ 0.566
& \textbf{$0.288\!\pm\!0.003$}
& $\infty$
& 0.436 $\pm$ 0.068
& 0.716 $\pm$ 0.315
& 0.290 $\pm$ 0.015
& 2.336 & 0.575 & 0.375 \\

\midrule

20 & Lorenz-24
& \textbf{$6.653\!\pm\!0.952$}
& 7.517 $\pm$ 0.011
& $\infty$
& $\infty$
& $\infty$
& 7.566 $\pm$ 0.335
& 0.885 & 0.049 & 0.294 \\

20 & Lorenz-28
& 1.240 $\pm$ 0.617
& \textbf{$1.234\!\pm\!0.001$}
& 2.217 $\pm$ 1.828
& 1.521 $\pm$ 1.094
& $\infty$
& $\infty$
& 1.005 & 0.066 & 0.294 \\

20 & SpringChain-32
& 1.212 $\pm$ 0.148
& 1.195 $\pm$ 0.002
& 1.202 $\pm$ 0.013
& \textbf{$1.193\!\pm\!0.008$}
& $\infty$
& 1.196 $\pm$ 0.005
& 1.016 & 0.481 & 0.346 \\

20 & SpringChain-48
& 1.179 $\pm$ 0.336
& \textbf{$1.136\!\pm\!4.81{\times}10^{-4}$}
& $\infty$
& $\infty$
& $\infty$
& $\infty$
& 1.039 & 0.382 & 0.375 \\

20 & Burgers-32
& \textbf{$0.638\!\pm\!0.179$}
& 2.095 $\pm$ 0.011
& 3.311 $\pm$ 3.652
& 2.248 $\pm$ 0.092
& 3.132 $\pm$ 0.541
& 2.099 $\pm$ 0.028
& 0.304 & 0.064 & 0.320 \\

20 & Burgers-48
& \textbf{$0.631\!\pm\!0.199$}
& 1.988 $\pm$ 0.006
& 2.201 $\pm$ 0.166
& 2.168 $\pm$ 0.184
& $\infty$
& 1.991 $\pm$ 0.021
& 0.317 & 0.072 & 0.334 \\

20 & AdvecDiff-32
& \textbf{$1.042\!\pm\!0.957$}
& 1.910 $\pm$ 0.012
& 1.901 $\pm$ 0.201
& 1.906 $\pm$ 0.057
& $\infty$
& 1.929 $\pm$ 0.039
& 0.548 & 0.383 & 0.319 \\

20 & AdvecDiff-48
& \textbf{$0.775\!\pm\!0.520$}
& 1.899 $\pm$ 0.013
& 1.823 $\pm$ 0.102
& 1.855 $\pm$ 0.042
& $\infty$
& 1.900 $\pm$ 0.032
& 0.425 & 0.288 & 0.333 \\

20 & ReactionDiff-32
& \textbf{$0.134\!\pm\!0.116$}
& 0.280 $\pm$ 0.006
& $\infty$
& $\infty$
& 1.532 $\pm$ 0.606
& 0.289 $\pm$ 0.028
& 0.480 & 0.383 & 0.320 \\

20 & ReactionDiff-48
& \textbf{$0.234\!\pm\!0.359$}
& 0.291 $\pm$ 0.005
& 1.119 $\pm$ 1.061
& $\infty$
& 2.341 $\pm$ 1.040
& 0.302 $\pm$ 0.016
& 0.805 & 0.433 & 0.334 \\

\midrule

40 & Lorenz-24
& \textbf{$6.170\!\pm\!0.953$}
& 7.516 $\pm$ 0.010
& $\infty$
& $\infty$
& $\infty$
& 7.559 $\pm$ 0.207
& 0.821 & 0.285 & 0.292 \\

40 & Lorenz-28
& 0.920 $\pm$ 0.545
& 1.235 $\pm$ 0.001
& $\infty$
& \textbf{$0.831\!\pm\!0.942$}
& 2.337 $\pm$ 2.884
& $\infty$
& 1.106 & 0.064 & 0.290 \\

40 & SpringChain-32
& \textbf{$1.070\!\pm\!0.379$}
& 1.193 $\pm$ 0.001
& $\infty$
& 1.194 $\pm$ 0.007
& $\infty$
& 1.195 $\pm$ 0.004
& 0.897 & 0.389 & 0.320 \\

40 & SpringChain-48
& \textbf{$1.087\!\pm\!0.017$}
& 1.135 $\pm 5.42{\times}10^{-4}$
& $\infty$
& $\infty$
& $\infty$
& $\infty$
& 0.958 & 0.575 & 0.334 \\

40 & Burgers-32
& \textbf{$0.537\!\pm\!0.155$}
& 2.096 $\pm$ 0.020
& $\infty$
& 2.457 $\pm$ 0.253
& $\infty$
& 2.120 $\pm$ 0.061
& 0.256 & 0.055 & 0.306 \\

40 & Burgers-48
& \textbf{$0.550\!\pm\!0.122$}
& 1.979 $\pm$ 0.015
& $\infty$
& 2.107 $\pm$ 0.153
& $\infty$
& 1.973 $\pm$ 0.030
& 0.279 & 0.059 & 0.311 \\

40 & AdvecDiff-32
& \textbf{$0.962\!\pm\!0.596$}
& 1.905 $\pm$ 0.016
& $\infty$
& 1.883 $\pm$ 0.064
& $\infty$
& 1.905 $\pm$ 0.031
& 0.511 & 0.221 & 0.303 \\

40 & AdvecDiff-48
& \textbf{$0.564\!\pm\!0.318$}
& 1.910 $\pm$ 0.016
& $\infty$
& 1.854 $\pm$ 0.033
& $\infty$
& 1.916 $\pm$ 0.030
& 0.304 & 0.238 & 0.311 \\

40 & ReactionDiff-32
& \textbf{$0.065\!\pm\!0.042$}
& 0.277 $\pm$ 0.006
& $\infty$
& $\infty$
& $\infty$
& 0.284 $\pm$ 0.022
& 0.235 & 0.221 & 0.306 \\

40 & ReactionDiff-48
& \textbf{$0.042\!\pm\!0.031$}
& 0.291 $\pm$ 0.005
& $\infty$
& $\infty$
& $\infty$
& 0.301 $\pm$ 0.020
& 0.143 & 0.238 & 0.313 \\

\bottomrule
\end{tabular}}
\end{table*}





\section{Additional Optical Case Studies}
\label{sec:additional_optical_cases}

Figure~\ref{fig:lln_error_case_bank_appendix} summarizes the same fixed model, full-test PCA basis, and deterministic confidence ordering for all seven LLN errors on the two-moons test set. It places the error locations, overlaid LLN position paths, layerwise turn-angle matrix, and full-test turn distribution in one large overview, so the appendix supports the main-text case without repeating seven small case panels.

\begin{table}[!t]
\centering
\small
\caption{\textbf{Complete index of LLN errors in the fixed two-moons
case study.}
Errors are written as true class \(\rightarrow\) predicted class.
Rank orders the errors by increasing predicted-class confidence, and
the peak layer maximizes \(\Delta\alpha_l\).}
\label{tab:lln_complete_error_index}
\setlength{\tabcolsep}{6pt}
\renewcommand{\arraystretch}{1.08}
\begin{tabular*}{\textwidth}{
@{\extracolsep{\fill}}
cccccc
@{}}
\toprule
Rank
& Test index
& Error
& Confidence
& Max. turn (\(^\circ\))
& Peak layer \\
\midrule
0 & 183 & \(1\!\rightarrow\!0\) & 0.547 & 16.62 & L1 \\
1 & 94  & \(0\!\rightarrow\!1\) & 0.560 & 18.52 & L1 \\
2 & 260 & \(1\!\rightarrow\!0\) & 0.592 & 17.38 & L1 \\
3 & 199 & \(1\!\rightarrow\!0\) & 0.624 & 13.93 & L1 \\
4 & 107 & \(0\!\rightarrow\!1\) & 0.661 & 16.11 & L7 \\
5 & 366 & \(0\!\rightarrow\!1\) & 0.726 & 17.31 & L7 \\
6 & 202 & \(1\!\rightarrow\!0\) & 0.788 & 16.96 & L1 \\
\bottomrule
\end{tabular*}
\end{table}

\begin{figure}
\centering
\includegraphics[width=0.9\textwidth]{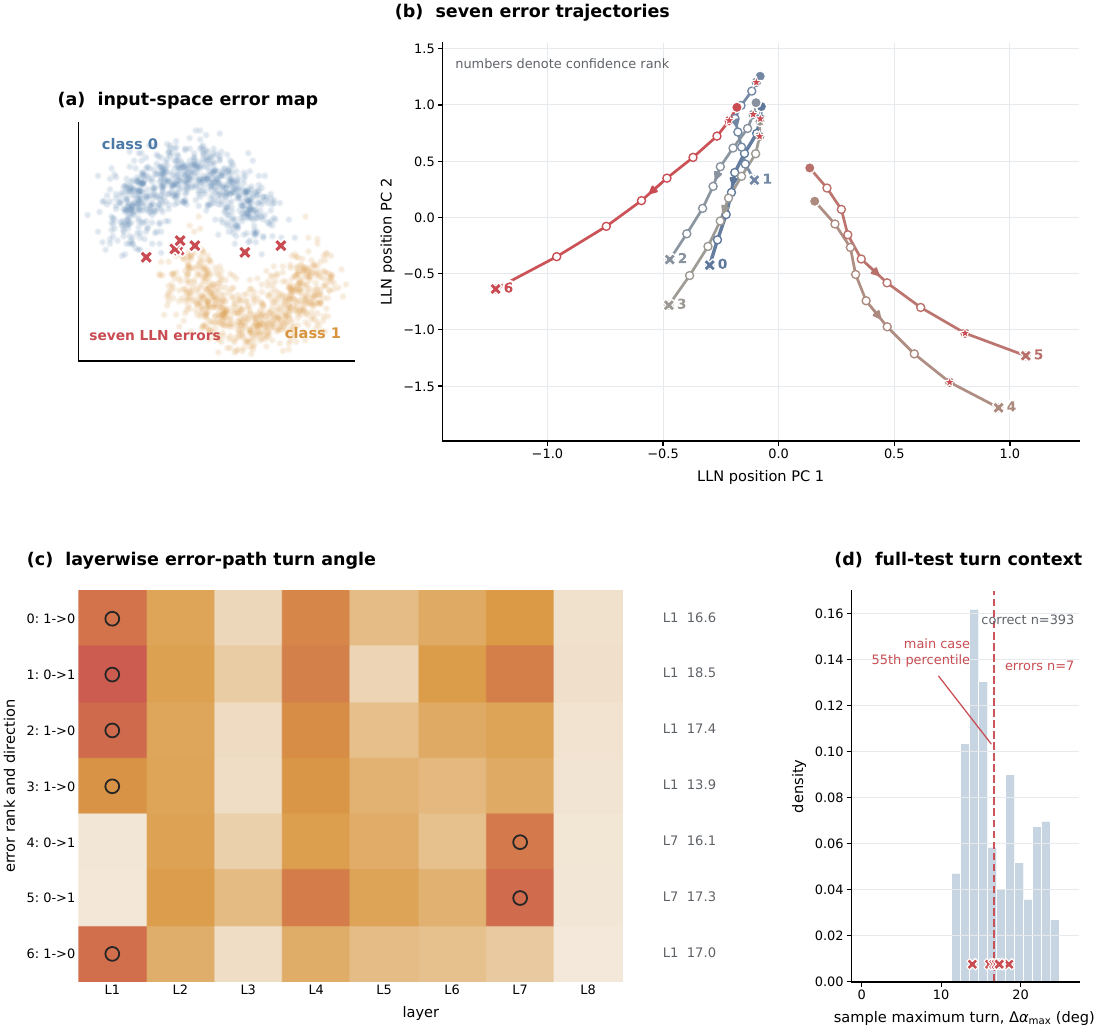}
\captionof{figure}{\textbf{Large overview of all seven LLN error paths and turn statistics.} (a) Input-space locations of the seven deterministic LLN errors. (b) LLN position trajectories in the full-test PCA basis; numbers denote confidence rank and stars mark each error path's peak-turn layer. (c) Layerwise wrong-sample turn angles for all error ranks. (d) Full-test maximum-turn distribution comparing LLN errors with correct samples.}
\label{fig:lln_error_case_bank_appendix}
\end{figure}

\end{document}